\documentclass[10pt]{article}
\usepackage{amsmath,amssymb, mathrsfs}
\usepackage{amsthm}
\usepackage{times}
\usepackage{graphicx}
\usepackage{color}
\usepackage{multirow}
\usepackage[round]{natbib}
\usepackage{rotating}
\usepackage{bbm}
\usepackage{latexsym}

\usepackage{algorithm}
\usepackage{algorithmicx}
\usepackage{algpseudocode}

\usepackage{subfig}

\newtheorem{lemma}{Lemma}[section]
\newtheorem{proposition}{Proposition}[section]
\newtheorem{theorem}{Theorem}[section]
\newtheorem{assumption}{Assumption}[section]

\title{Risk-sensitive Reinforcement Learning\footnote{Accepted for publication in \emph{Neural Computation.}}}

\author{Yun Shen$^{\displaystyle 1}$, Michael J.\ Tobia$^{\displaystyle 3}$, Tobias Sommer$^{\displaystyle 3}$, Klaus Obermayer$^{\displaystyle 1, \displaystyle 2}$ \\
{$^{\displaystyle 1}$Technical University Berlin, Germany}\\
{$^{\displaystyle 2}$Bernstein Computational Neuroscience Center, Berlin, Germany}\\
{$^{\displaystyle 3}$}University Medical Center Hamburg-Eppendorf, Germany}

\date{January 23, 2014}

\begin{document}

\maketitle

\begin{abstract}

We derive a family of risk-sensitive reinforcement learning methods for agents, who face sequential decision-making tasks in uncertain environments. By applying a utility function to the temporal difference (TD) error, nonlinear transformations are effectively applied not only to the received rewards but also to the true transition probabilities of the underlying Markov decision process. When appropriate utility functions are chosen, the agents' behaviors express key features of human behavior as predicted by prospect theory \citep{kahneman1979prospect}, for example different risk-preferences for gains and losses as well as the shape of subjective probability curves. We derive a risk-sensitive Q-learning algorithm, which is necessary for modeling human behavior when transition probabilities are unknown, and prove its convergence. As a proof of principle for the applicability of the new framework we apply it to quantify human behavior in a sequential investment task. We find, that the risk-sensitive variant provides a significantly better fit to the behavioral data and that it leads to an interpretation of the subject's responses which is indeed consistent with prospect theory. The analysis of simultaneously measured fMRI signals show a significant correlation of the risk-sensitive TD error with BOLD signal change in the ventral striatum. In addition we find a significant correlation of the risk-sensitive Q-values with neural activity in the striatum, cingulate cortex and insula, which is not present if standard Q-values are used. 

\end{abstract}

\section{Introduction}
Risk arises from the uncertainties associated with future events, and is inevitable since the consequences of  actions are uncertain at the time when a decision is made. Hence, risk has to be taken into account by the decision-maker, consciously or unconsciously. An economically rational decision-making rule, which is \emph{risk-neutral}, is to select the alternative with the highest expected reward. In the context of sequential or multistage decision-making problems, \emph{reinforcement learning} (RL, \citealp{sutton1998reinforcement}) follows this line of thought. It describes how an agent ought to take actions that maximize expected cumulative rewards in an environment typically described by a \emph{Markov decision process} (MDP, \citealp{puterman1994markov}).
RL is a well-developed model not only for human decision-making, but also for models of free choice in non-humans, because similar computational structures, such as dopaminergically mediated reward prediction errors, have been identified across species (\citealp{schultz1997neural,schultz2002getting}). 

Besides risk-neutral policies, \emph{risk-averse} policies, which accept a choice with a more certain but possibly lower expected reward, are also considered economically rational \citep{gollier2004economics}. For example, a risk-averse investor might choose to put money into a bank account with a low but guaranteed interest rate, rather than into a stock with possibly high expected returns but also a chance of high losses. Conversely, \emph{risk-seeking} policies, which prefer a choice with less certain but possibly high reward, are considered economically irrational. 
Human agents are, however, not always economically rational \citep{gilboa2009theory}. Behavioral studies show that human can be risk-seeking in one situation while risk-averse in another situation  \citep{kahneman1979prospect}. RL algorithms developed so far cannot effectively model these complicated risk-preferences.

Risk-sensitive decision-making problems, in the context of MDPs, have been investigated in various fields, e.g., in machine learning \citep{heger1994consideration, mihatsch2002risk}, optimal control \citep{hernandez1996risk}, operations research \citep{howard1972risk, borkar2002q}, finance \citep{ruszczynski2010risk}, as well as neuroscience \citep{nagengast2010risk,braun2011risk,niv2012neural}. 
Note that the core of MDPs consists of two sets of \emph{objective} quantities describing the environment: immediate \emph{rewards} obtained at states by executing actions, and \emph{transition probabilities} for switching states when performing actions. Facing the same environment, however, different agents might have different policies, which indicates that risk is taken into account differently by different agents. Hence, to incorporate risk, which is derived from both quantities, all existing literature applies a nonlinear transformation to either the experienced reward values or to the transition probabilities, or to both. The former is the canonical approach in classical economics, as in expected utility theory \citep{gollier2004economics}, while the latter originates from behavioral economics, as in \emph{subjective probability} \citep{savage1972foundations}, but is also derived from a rather recent development in mathematical finance, \emph{convex/coherent risk measures} \citep{artzner1999coherent, follmer2002convex}. For modeling human behaviors, prospect theory (\citealp{kahneman1979prospect}) suggests that we should combine both approaches, i.e., human beings have different perceptions not only for the same objective amount of rewards but also the same value of the true probability.  
Recently, \cite{niv2012neural} combined both approaches by applying piecewise linear functions (an approximation of a nonlinear transformation) to reward prediction errors that contain the information of rewards directly and the information of transition probabilities indirectly. Importantly, the reward prediction errors that incorporated experienced risk were strongly coupled to activity in the nucleus accumbens of the ventral striatum, providing a biologically based plausibility to this combined approach. In this work we show (in Section 2.1) that the risk-sensitive algorithm proposed by Niv and colleagues is a special case of our general risk-sensitive RL framework.

Most of the literature in economics or engineering fields focuses on economically rational risk-averse/-neutral strategies, which are not always adopted by humans. The models proposed in behavioral economics, despite allowing economic irrationality, require knowledge of the true probability, which usually is not available at the outset of a learning task. In neuroscience, on the one hand, several works (e.g., \citealp{wu2009economic, preuschoff2008human}) follow the same line as in behavioral economics and require knowledge of the true probability. On the other hand, though different modified RL algorithms (e.g., \citealp{glimcher2008neuroeconomics, symmonds2011deconstructing}) are applied to model human behaviors in learning tasks, the algorithms often fail to generalize across different tasks. In our previous work \citep{Shen2013}, we described a general framework for incorporating risk into MDPs by introducing nonlinear transformations to both rewards and transition probabilities. A risk-sensitive objective was derived and optimized by value iteration or dynamic programming. This solution, hence, does not work in learning tasks where the true transition probabilities are unknown to learning agents. For this purpose, a model-free framework for RL algorithms is to be derived in this paper, where, similar to Q-learning, the knowledge of the transition and reward model is not needed.

This paper is organized as follows. 
Section \ref{sec:valfunc} starts with a mathematical introduction into \emph{valuation functions} for measuring risk. We then specify a sufficiently rich class of valuation functions in Section \ref{sec:ubsf} and provide the intuition behind our approach by applying this class to a simple example in Section \ref{sec:toyex}. We aslo show that key features of prospect theory can be captured by this class of valuation functions. Restricted to the same class, we derive a general framework for risk-sensitive Q-learning algorithms and prove its convergence in Section \ref{sec:rsmdps}. Finally, in Section \ref{sec:ex}, we apply this framework to quantify human behavior. We show that the risk-sensitive variant provides a significantly better fit to the behavioral data and significant correlations are found between sequences generated by the proposed framework and changes of fMRI BOLD signals.

\section{Valuation Functions and Risk Sensitivities}
\label{sec:valfunc}
Suppose that we are facing choices. Each \emph{choice} might yield different outcomes when events are generated by a random process. Hence, to keep generality, we model the outcome of each choice by a real-valued random variable $\{ X(i), \mu(i) \}_{i \in I}$, where $I$ denotes an \emph{event space} with a finite cardinality $\lvert I \rvert$ and $X(i) \in \mathbb R$ is the outcome of $i$th event with probability $\mu(i)$. 
We say two vectors $X \leq Y$ if $X(i) \leq Y(i)$ for all $i \in I$. Let $\mathbf 1$ (resp.~$\mathbf 0$) denote the vector with all elements equal 1 (resp.~0). Let $\mathscr P$ denote the space of all possible distributions $\mu$.

Choices are made according to their outcomes. Hence, we assume that there exists a mapping $\rho: \mathbb R^{\lvert I \rvert} \times \mathscr P \rightarrow \mathbb R$ such that  one prefers $(X,\mu)$ to $(Y, \nu)$ whenever $\rho(X,\mu) \geq \rho(Y,\nu)$. We assume further that $\rho$ satisfies the following axioms inspired by the \emph{risk measure theory} applied in mathematical finance \citep{artzner1999coherent, follmer2002convex}. A mapping $\rho: \mathbb R^{\lvert I \rvert} \times \mathscr P \rightarrow \mathbb R$ is called a \textbf{valuation function}, if it satisfies for each $\mu \in \mathscr P$,
\begin{itemize}
 \item[I]  (monotonicity) $\rho(X, \mu) \leq \rho(Y, \mu)$, whenever $X\leq Y \in \mathbb R^{\lvert I \rvert}$; 
 \item[II]  (translation invariance) $\rho(X + y \mathbf 1, \mu) = \rho(X, \mu) + y$, for any $y \in \mathbb R$.
\end{itemize}

Within the economic context, $X$ and $Y$ are outcomes of two choices. Monotonicity reflects the intuition that given the same event distribution $\mu$, if the outcome of one choice is \emph{always} (for all events) 
higher than the outcome of another choice, the \emph{valuation} of the choice must be also higher. Under the axiom of translation invariance, the sure outcome $y \mathbf 1$ (equal outcome for every event) after executing decisions, is considered as a sure outcome before making decision. This also reflects the intuition that there is no risk if there is no uncertainty.

In our setting, valuation functions are not necessarily centralized, i.e.\ $\rho(\mathbf 0, \mu)$ is not necessarily 0, since $\rho(\mathbf 0, \mu)$ in fact sets a reference point, which can differ for different agents. However, we can centralize any valuation function by $\tilde \rho(X, \mu) := \rho (X, \mu) - \rho(\mathbf 0, \mu)$. From the two axioms, it follows that (for the proof see Lemma \ref{lm:inrm} in Appendix)
\begin{align}
  \min_{i \in I} X_i =: \underline{X} \leq \tilde \rho(X, \mu) \leq \overline{X}:= \max_{i \in I} X_i, \forall \mu \in \mathscr P, X \in \mathbb R^{\lvert I \rvert}. \label{eq:range}
\end{align}
$\overline{X}$ is the possibly largest outcome, which represents the most optimistic prediction of the future, while $\underline{X}$ is the possibly smallest outcome and the most pessimistic estimation. The centralized valuation function $\tilde \rho(X, \mu)$ satisfying $\tilde \rho(0,\mu) = 0$ can be in fact viewed as a subjective mean of the random variable $X$, which varies from the best scenario $\overline{X}$ to the worst scenario $\underline{X}$, covering the objective mean as a special case.

To judge the risk-preference induced by a certain type of valuation functions, we follow the rule that \emph{diversification} should be preferred if the agent is \emph{risk-averse}. More specifically, suppose an agent has two possible choices, one of which leads to the future reward $(X,\mu)$ while the other one leads to the future reward $(Y,\nu)$. For simplicity we assume $\mu = \nu$. If the agent \emph{diversifies}, i.e., if one spends only a fraction $\alpha$ of the resources on the first and the remaining amount on the second alternative, the future reward is given by $\alpha X+ (1-\alpha)Y$. If the applied valuation function is concave, i.e., 
$$\rho(\alpha X+ (1-\alpha)Y, \mu) \geq  \alpha \rho(X, \mu) + (1-\alpha)\rho(Y, \mu),$$
for all $\alpha \in [0,1]$ and $X, Y \in \mathbb R^{\lvert I \rvert},$
then the diversification should increase the (subjective) valuation. Thus, we call the agent's behavior \emph{risk-averse}. Conversely, if the applied valuation function is \emph{convex}, the induced risk-preference should be \emph{risk-seeking}.

\subsection{Utility-based Shortfall}
\label{sec:ubsf}
We now introduce a class of valuation functions, the utility-based shortfall, which generalizes many important special valuation functions in literature. Let $u: \mathbb R \rightarrow \mathbb R$ be a \emph{utility function}, which is continuous and strictly increasing. The shortfall $\rho_{x_0}^{u}$ induced by $u$ and an \emph{acceptance level} $x_0$ is then defined as
\begin{align}
 \rho_{x_0}^{u}(X, \mu) := \sup\left\{ m \in \mathbb R \ | \ \sum_{i \in I} u(X(i) - m) \mu(i) \geq x_0 \right\}, \label{eq:sh}
\end{align}
It can be shown (cf.~\citealp{follmer2004stochastic}) that $\rho_{x_0}^{\textrm{u}}$ is a valid valuation function satisfying the axioms. The utility-based shortfall was first introduced in the mathematical finance literature \citep{follmer2004stochastic}. The class of utility functions considered here will, however, be more general than the class of utility functions typically used in finance.

Comparing with the expected utility theory, the utility function in Eq.~\eqref{eq:sh} is applied to the relative value $X(i)-m$ rather than to the absolute outcome $X(i)$. This reflects the intuition that human beings judge utilities usually by comparing those outcome with a reference value
which may not be zero. 
The property of $u$ being convex or concave determines the risk sensitivity of $\rho_{x_0}^{u}$: given a concave function $u$, $\rho$ is also concave and hence risk-averse (see Theorem 4.61, \citealp{follmer2004stochastic}). Vice versa, $\rho$ is convex (hence risk-seeking) for convex $u$.

Utility-based shortfalls cover a large family of valuation functions, which have been proposed in literature of various fields. 
\begin{itemize}
 \item[(a)] For $u(x) = x$ and $x_0 = 0$, one obtains the standard expected reward $\rho(X, \mu) = \sum_{i} X(i) \mu(i)$.
 \item[(b)] For $u(x) = e^{\lambda x}$ and $x_0 = 1$, one obtains $\rho(X, \mu) = \frac{1}{\lambda} \log \left[ \sum_{i} \mu(i) e^{\lambda X(i)} \right]$ (the so called \emph{entropic map}, see e.g.~\citealp{cavazos2010optimality} and references therein). Expansion w.r.t.~$\lambda$ leads to
 \begin{align*}
  \rho(X, \mu) = \mathbb E^\mu [X] + \lambda \textrm{Var}^\mu[X] + O(\lambda^2)
 \end{align*}
where $\textrm{Var}^\mu[X]$ denotes the variance of $X$ under the distribution $\mu$. Hence, the entropic map is risk-averse if $\lambda<0$ and risk-seeking if $\lambda >0$. In neuroscience, \cite{nagengast2010risk} and \cite{braun2011risk} applied this type of valuation function to test risk-sensitivity in human sensorimotor control.
 \item[(c)] \cite{mihatsch2002risk} proposed the following setting 
\begin{align*}
 u(x) = \left\lbrace 
\begin{array}{ll}
 (1-\kappa) x & \textrm{ if } x > 0 \\
 (1 + \kappa) x & \textrm{ if } x \leq 0 
\end{array}
 \right., 
\end{align*}
where $\kappa \in (-1,1)$ controls the degree of risk sensitivity. Its sign determines the property of the utility function $u$ being convex vs.~concave and, therefore, the risk-preference of $\rho$. In a recent study, \cite{niv2012neural} applied this type of valuation function to quantify risk-sensitive behavior of human subjects and to interpret the measured neural signals.
\end{itemize}
When quantifying human behavior, combined convex/concave utility functions, e.g.,
\begin{align}
 u_p(x) = \left\{
 \begin{array}{ll}
  k_+ x^{l_+} & x  \geq 0 \\
  - k_- (-x)^{l_-} & x < 0
 \end{array}
 \right., 
 \label{eq:utilfunc}
\end{align}
are of special interest, since people tend to treat gains and losses differently and, therefore, have different risk preferences on gain and loss sides. In fact, the polynomial function in Eq.~\eqref{eq:utilfunc} was used in the prospect theory \citep{kahneman1979prospect} to model human risk preferences and the results show that $l_+$ is usually below 1, i.e., $u_p(x)$ is concave and thus risk-averse on gains, while $l_-$ is also below 1 and $u_p(x)$ is therefore convex and risk-seeking on losses.

\subsection{Utility-based Shortfall and Prospect Theory} \label{sec:toyex}
To illustrate the risk-preferences induced by different utility functions, we consider a simple example with two events. The first event has outcome $x_1$ with probability $p$, while the other event has smaller outcome $x_2 < x_1$ with $1-p$. Note that $p = \frac{\mathbb E X - x_2}{x_1 - x_2}$, where $\mathbb E X = p x_1 + (1-p) x_2$ denotes the risk-neutral mean. 

Replacing $\mathbb E X$ with the \emph{subjective mean} $\tilde \rho(X, p) = \rho(X, p) - \rho(0, p)$ defined in Eq.~\eqref{eq:range},  
we can define a \emph{subjective probability} (cf.\ \cite{tversky1992advances}) as 
\begin{align}
 w(p) := \frac{\tilde \rho(X, p) - x_2}{x_1 - x_2}, \label{eq:subprob}
\end{align}
which measures agents' subjective perception of the true probability $p$.  

In risk-neutral cases, $\tilde \rho(X, p)$ is simply the mean and $w(p) = p$. In risk-averse cases, the balance moves towards the worst scenario. Hence, the probability of the first event (with larger outcome $x_1$) is always underestimated. On the contrary, in risk-seeking cases, the probability of the first event is always overestimated. Behavioral 
studies show that human subjects usually overestimate low probabilities and underestimate high probabilities \citep{tversky1992advances}. This can be quantified by applying mixed valuation functions $\rho$. If we apply utility-based shortfalls, it can be quantified by using mixed utility function $u$.

\begin{figure}[ht]
 \centering
   \includegraphics[width=0.8\textwidth]{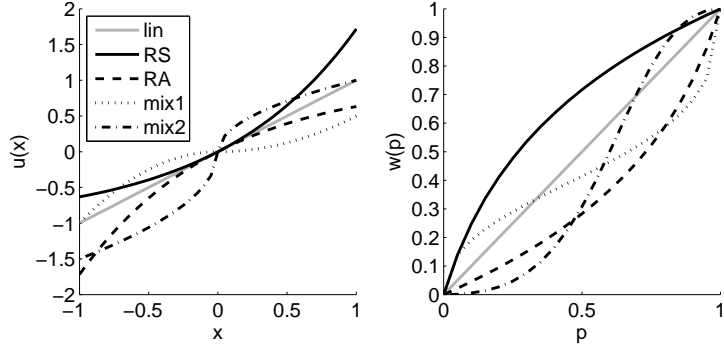}
\caption{Shortfalls with different utility functions and induced subjective probabilities. (Left) utility functions defined as follows: $\textrm{lin}: x; \textrm{RS}: e^{x} - 1; \textrm{RA}: 1 - e^{-x}$; $\textrm{mix1: } u_p(x)$ as defined in Eq.~\eqref{eq:utilfunc} with $k_+ = 0.5$, $l_+ = 2$, $k_- = 1$ and $l_- = 2$; mix2: same as mix1 but with $k_+ = 1$, $l_+ = 0.5$, $k_- = 1.5$ and $l_- = 0.5$. (Right) subjective probability functions calculated according to Eq.~\eqref{eq:subprob}.}
\label{fig:uf}
\end{figure}

Let $x_1 = 1$, $x_2 = -1$ and the acceptance level $x_0 = 0$. 
Fig.\ \ref{fig:uf} (left) shows five different utility functions, one linear function ``lin'', one convex function ``RS'', one concave function ``RA'', and two mixed functions ``mix1'' and ``mix2'' (for details see caption). The corresponding subjective probabilities are shown in Fig.\ \ref{fig:uf} (right). Since the function ``RA'' is concave, the corresponding valuation function is risk-averse 
and therefore the probability of high-reward event is always underestimated. For the case of the convex function ``RS'', the probability of high-reward event is always overestimated. However, since the ``mix1'' function is convex on $[0,\infty)$ but concave on $(-\infty,0]$, high probabilities are underestimated while low probabilities are overestimated, which replicates very well the probability weighting function applied in prospect theory for gains (cf.~Fig.\ 1, \citealp{tversky1992advances}). Conversely, the ``mix2'' function, which is concave on $[0,\infty)$ and convex on $(-\infty,0]$, corresponds to the overestimation of high probabilities and the underestimation of low probabilities. This corresponds to the weighting function used for losses in prospect theory (cf.~Fig.\ 2, \citealp{tversky1992advances}). 

We will see in the following section that the advantage of using the utility-based shortfall is that we can derive iterating learning algorithms for the estimation of the subjective valuations, whereas it is difficult to derive such algorithms in the framework of prospect theory.

\section{Risk-sensitive Reinforcement Learning}
\label{sec:rsmdps}
A Markov decision process (see e.g.\ \citealt{puterman1994markov}) $$\mathcal M = \{ \mathbf S, (\mathbf A, \mathbf A(s), s\in \mathbf S), \mathcal P, (r, \mathcal P_r) \},$$ consists of a state space $\mathbf S$, admissible action spaces $\mathbf A(s) \subset \mathbf A$ at $s \in \mathbf S$, a transition kernel $\mathcal P(s'|s,a)$, which denotes the transition probability moving from one state $s$ to another state $s'$ by executing action $a$, and a reward function $r$ with its distribution $\mathcal P_r$. In order to model random rewards, we assume that the reward function has the form\footnote{In standard MDPs, it is sufficient \citep{puterman1994markov} to consider the \emph{deterministic} reward function $\bar r(s,a) := \sum_{\epsilon \in \mathbf E} r(s,a,\varepsilon) \mathcal P_r(\varepsilon|s,a)$, i.e., the mean reward at each $(s,a)$-pair. In risk-sensitive cases, random rewards cause also risk and uncertainties. Hence, we keep the generality by using random rewards.} 
$$r(s,a,\varepsilon):\mathbf S \times \mathbf A \times \mathbf E \rightarrow \mathbb R.$$ 
$\mathbf E$ denotes the noise space with distribution $\mathcal P_r(\varepsilon|s,a)$, i.e., given $(s,a)$, $r(s,a,\varepsilon)$ is a random variable with values drawn from $\mathcal P_r(\cdot|s,a)$. Let $R(s,a)$ be the \emph{random} reward gained at $(s,a)$, which follows the distribution $\mathcal P_r(\cdot|s,a)$.  The random state (resp.~action) at time $t$ is denoted by $S_t$ (resp.~$A_t$). Finally, we assume that all sets $\mathbf S, \mathbf A, \mathbf E$ are finite.

A \emph{Markov policy} $\boldsymbol \pi = [\pi_0, \pi_1, \ldots]$ consists of a sequence of single-step Markov policies at times $t=0,1,\ldots$, where $\pi_t(A_t = a|S_t = s)$ denotes the probability of choosing action $a$ at state $s$. Let $\Pi$ be the set of all Markov policies. The optimal policy within a time horizon $T$ is obtained by maximizing the expectation of the discounted cumulative rewards,
\begin{align}
 J_T(\boldsymbol \pi, s) := \max_{\boldsymbol \pi \in \Pi} \mathbb E \left[ \sum_{t=0}^T \gamma^t R(S_t,A_t) | S_0 = s, \boldsymbol \pi\right]. \label{eq:mdp:obj}
\end{align}
where $s \in \mathbf S$ denotes the initial state and $\gamma \in [0,1)$ the discount factor. Expanding the sum leads to 
\begin{align}
 J_T(\boldsymbol \pi, s) = \mathbb E^{\pi_0}_{S_0=s} \left[R(S_0,A_0) +  \gamma 
   \mathbb E^{\pi_1}_{S_1} \left[ R(S_1,A_1) +  \ldots + \gamma \mathbb
     E^{\pi_{T}}_{S_{T}} \left[ R(S_T,A_T) \right] \ldots  \right]
 \right]. \label{eq:T}
\end{align}
We now generalize the conditional expectation $\mathbb E_{s}^\pi$ to represent the valuation functions considered in Section \ref{sec:valfunc}. Let $\mathbf K := \{ (s,a) | s \in \mathbf S, a \in \mathbf A(s)\}$ be the set of all admissible state-action pairs. Let
\begin{align}
 I = \mathbf S \times \mathbf E \quad \textrm{and} \quad \mu_{s,a}(s',\varepsilon) = \mathcal P(s'|s,a) \mathcal P_r(\varepsilon|s,a). \label{eq:mu}
\end{align}
A mapping $\mathcal U(X, \mu|s,a): \mathbb R^{\lvert I \rvert} \times \mathscr P \times \mathbf K \rightarrow \mathbb R$ is called a \textbf{valuation map}, if for each $(s,a) \in \mathbf K$, $\mathcal U(\cdot|s,a)$ is a valuation function on $\mathbb R^{\lvert I \rvert} \times \mathscr P$. 
Let $\mathcal U_{s,a}(X, \mu)$ be a short notation of $\mathcal U(X, \mu|s,a)$ and let $$\mathcal U^\pi_s(X, \mu) := \sum_{a \in \mathbf A(s)} \pi(a | s) \mathcal U(X, \mu|s,a)$$
be the valuation map averaged over all actions.
Since $\mu \equiv \mu_{s,a}$ for each $(s,a) \in \mathbf K$, we will omit $\mu$ in $\mathcal U$ in the following.
Replacing the conditional expectation $\mathbb E_s^\pi$ with $\mathcal U^\pi_s$ in Eq.~\eqref{eq:T}, the risk-sensitive objective becomes
\begin{align}
 \tilde J_T(\boldsymbol \pi, s) := \mathcal U^{\pi_0}_{S_0=s} [R(S_0,A_0) +  \gamma 
   \mathcal U^{\pi_1}_{S_1} [ R(S_1,A_1) +  \ldots + \gamma \mathcal
     U^{\pi_{T}}_{S_{T}} \left[ R(S_T,A_T) \right] \ldots  ] ].\label{eq:T:rs}
\end{align}
The optimal policy is then given by $\max_{\boldsymbol \pi \in \Pi} \tilde J_T(\boldsymbol \pi, s)$. For infinite-horizon problem, we obtain
\begin{align}
 \max_{\boldsymbol \pi \in \Pi} \tilde J(\boldsymbol \pi, s) := \lim_{T \rightarrow \infty}\tilde J_T(\boldsymbol \pi, s), \label{eq:obj:rs}
\end{align}
using the same line of argument.

The optimization problem for finite-stage objective function $\tilde J_T$ can be solved by a generalized \emph{dynamic programming} \citep{bertsekas1996neuro}, while the one defined in Eq.~\eqref{eq:obj:rs} requires the solution to the \emph{risk-sensitive Bellman equation}:
\begin{align}
 V^*(s) = \max_{a \in \mathbf A(s)} \mathcal U_{s,a}(R(s,a) + \gamma V^*). \label{eq:bellman}
\end{align}
The latter is a consequence of the following theorem.
\begin{theorem}[Theorem 5.5, \citealp{Shen2013}]
\label{th:vi} $V^*(s)= \max_{\boldsymbol \pi} \tilde J(\boldsymbol \pi, s)$ holds for all $s \in \mathbf S$, whenever $V^*$ satisfies the equation \eqref{eq:bellman}. Furthermore, a deterministic policy $\pi^*$ is optimal, if $\pi^*(s) = \arg\max_{a \in \mathbf A(s)} \mathcal U_{s,a}(R + \gamma V^*)$. 
\end{theorem}
Define $Q^*(s,a) := \mathcal U_{s,a}(R + \gamma V^*)$. Then Eq.~\eqref{eq:bellman} becomes
\begin{align}
 Q^*(s,a)= \mathcal U_{s,a}\left(R(s,a) + \gamma  \max_{a \in \mathbf A(s')} Q^*(s',a)\right), \forall (s,a) \in \mathbf K. \label{eq:optimalQ}
\end{align}

To carry out value iteration algorithms, the MDP $\mathcal M$ must be known \emph{a priori}. In many real-life situations, however, the transition probabilities are unknown as well as the outcome of an action before its execution. Therefore, an agent has to explore the environment while gradually improving its policy.
We now derive RL-type algorithms for estimating Q-values of general valuation maps based on the utility-based shortfall, which do not require knowledge of the reward and transition model. 

\begin{proposition}[cf.~Proposition 4.104, \citealp{follmer2004stochastic}] \label{prop:impl} 
 Let $\rho_{x_0}^{\textrm{u}}$ be a shortfall defined in Eq.~\eqref{eq:sh}, where $u$ is continuous and strictly increasing. Then the following statements are equivalent: (i) $\rho_{x_0}^{\textrm{u}}(X) = m^*$ and (ii) $\mathbb E^\mu[u(X-m^*)] = x_0$.
\end{proposition}
\noindent For proof see Appendix A.

Consider the valuation map induced by the utility-based shortfall\footnote{In principle, we can apply different utility functions $u$ and acceptance levels $x_0$ at different $(s,a)$-pairs. However, for simplicity, we
drop their dependence on $(s,a)$.} $$\mathcal U_{s,a}(X) = \sup\{m \in \mathbb R \ | \ \mathbb E^{\mu_{s,a}}\left[u(X-m)\right] \geq x_0 \},$$ where $\mu_{s,a}$ is defined in Eq.~\eqref{eq:mu}. 
If $\mathcal U_{s,a}(X) = m^*(s,a)$ exists, Proposition \ref{prop:impl} assures that $m^*(s,a)$ is the unique solution to equation $$\mathbb E^{\mu_{s,a}}\left[u(X-m^*(s,a))\right] = x_0.$$ Let $X = R + \gamma V^*$. Then $m^*(s,a)$ corresponds to the optimal Q-value $Q^*(s,a)$ defined in Eq.~\eqref{eq:optimalQ}, which is equivalent to  
\begin{align}
\sum_{s' \in \mathbf S, \varepsilon \in \mathbf E} \mathcal P(s'|s,a) \mathcal P_r(\varepsilon|s,a)
u \left( r(s,a,\varepsilon) + \gamma \max_{a' \in \mathbf A(s')} Q^*(s',a') - Q^*(s,a) \right) \nonumber \\
= x_0, \forall (s,a) \in \mathbf K.  \label{eq:opt}
\end{align}
Let $\left\{s_t,a_t, s_{t+1}, r_t\right\}$ be the sequence of states, chosen actions, successive states and received rewards. Analogous to the standard Q-learning algorithm, we consider the following iterative procedure
\begin{align}
 Q_{t+1}(s_t,a_t) = Q_t(s_t,a_t) + \alpha_t(s_t,a_t) \left[ u\left(r_t + \gamma \max_{a} Q_t(s_{t+1},a) - Q_t(s_t,a_t)\right) - x_0 \right], \label{eq:ql}
\end{align}
where $\alpha_t \geq 0$ denotes learning rate function that satisfies $\alpha_t(s,a) > 0$ only if $(s,a)$ is updated at time $t$, i.e., $(s,a) = (s_t,a_t)$. In other words, for all $(s,a)$ that are not visited at time $t$, $\alpha_t(s,a) = 0$ and their Q-values are not updated.  
Consider utility functions $u$ with the following properties.
\begin{assumption}\label{ass:2}
 (i) The utility function $u$ is strictly increasing and there exists some $y_0\in \mathbb R$ such that $u(y_0) = x_0$. 
 (ii) There exist positive constants $\epsilon, L$ such that $ 0 < \epsilon \leq \frac{u(x) - u(y)}{x-y} \leq L $, for all $x \neq y \in \mathbb R$.
\end{assumption}
\noindent Then the following theorem holds (for proof see Appendix \ref{sec:proof_rl}).

\begin{theorem}\label{th:ql}
 Suppose Assumption \ref{ass:2} holds. Consider the generalized Q-learning algorithm stated in Eq.~\eqref{eq:ql}. If the nonnegative learning rates $\alpha_t(s,a)$ satisfy
 \begin{equation}
  \sum_{t=0}^\infty \alpha_t(s,a) = \infty \quad \textrm{ and } \quad \sum_{t=0}^\infty \alpha_t^2(s,a) < \infty, \quad \forall (s,a) \in \mathbf K, \label{eq:io}
 \end{equation}
then $Q_t(s,a)$ converges to $Q^*(s,a)$ for all $(s,a) \in \mathbf K$ with probability 1. 
\end{theorem}

The assumption in Eq.\ \eqref{eq:io} requires in fact that all possible state-action pairs must be visited infinitely often. Otherwise, the first sum in Eq.\ \eqref{eq:io} would be bounded by the setting of the learning rate function $\alpha_t(s,a)$. It means that, similar to the standard Q-learning, the agent has to explore the whole state-action space for gathering sufficient information about the environment. Hence, it can not take a too greedy policy in the learning procedure before the state-action space is well explored. We call a policy \textbf{proper} if under such policy every state is visited infinitely often. A typical policy, which is widely applied in RL literature as well as in models of human reward-based learning, is given by
 \begin{align}
   a_t \sim p(a_t | s_t) :=  \frac{e^{\beta Q(s_t,a_t)}}{\sum_{a} e^{\beta Q(s_t,a)}}, \label{eq:softmax}
 \end{align}
where $\beta \in [0,\infty)$ controls how greedy the policy should be. In Appendix \ref{sec:softmax}, we prove that under some technical assumptions upon the transition kernel of the underlying MDP, this policy is always proper.  
A widely used setting satisfying both conditions in Eq.~\eqref{eq:io} is to let $\alpha_t(s,a) := \frac{1}{N_t(s,a)}$, where $N_t(s,a)$ counts the number of times of visiting the state-action pair $(s,a)$ up to time $t$ and is updated trial-by-trial. This leads to the learning procedure shown in Algorithm \ref{alg:ql} (see also Fig.\ \ref{fig:rl}).
\begin{algorithm}[ht]                      
\caption{Risk-sensitive Q-learning}          
\label{alg:ql}                         
\begin{algorithmic}
\State initialize $Q(s,a) = 0$ and $N(s,a) = 0$ for all $s,a$.
\For{$t=1$ to $T$}
 \State at state $s_t$ choose action $a_t$ randomly using a proper policy (e.g.\ Eq.~\eqref{eq:softmax});
\State observe date $(s_t,a_t,r_t, s_{t+1})$;
 \State $N(s_t,a_t) \Leftarrow N(s_t,a_t) + 1$ and set learning rate: $\alpha_t := 1/N(s_t,a_t)$;
 \State update $Q$ as in Eq.~\eqref{eq:ql};
\EndFor
\end{algorithmic}
\end{algorithm}

The expression 
\begin{align*}
 TD_t := r_t + \gamma \max_{a} Q_t(s_{t+1},a) - Q_t(s,a)
\end{align*}
 inside the utility function of Eq.~\eqref{eq:ql} corresponds to the standard temporal difference (TD) error. Comparing Eq.~\eqref{eq:ql} with the standard Q-learning algorithm, we find that the nonlinear utility function is applied to the TD error (cf.~Fig.\ \ref{fig:rl}).  
This induces nonlinear transformation not only of the true rewards but also of the true transition probabilities, as has been shown in Section \ref{sec:ubsf}. By applying S-shape utility function, which is partially convex and partially concave,  we can therefore replicate key effects of prospect theory without the explicit introduction of a probability-weighting function.

\begin{figure}[ht]
 \centering
   \includegraphics[width=0.5\textwidth]{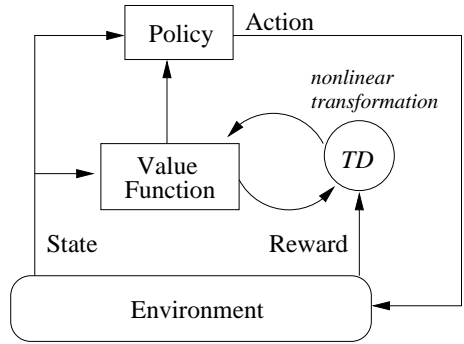}
\caption{Illustration of risk-sensitive Q-learning (cf.~Algorithm \ref{alg:ql}). The value function $Q(s,a)$ quantifies the current subjective evaluation of each state-action pair $(s,a)$. The next action is then randomly chosen according to a proper policy (e.g.~Eq.~\eqref{eq:softmax}) which is based on the current values of $Q$. After interacting with the environment, the agent obtains the reward $r$ and moves to the successor $s'$. The value function $Q(s,a)$ is then updated by the rule given in Eq.~\eqref{eq:ql}. This procedure continues until some stopping criterion is satisfied.}
\label{fig:rl}
\end{figure}

Assumption \ref{ass:2} (ii) seems to exclude several important types of utility functions. The exponential function $u(x) = e^x$ and the polynomial function $u(x) = x^p$, $p>0$, for example, do not satisfy the global Lipschitz condition required in Assumption \ref{ass:2} (ii). This problem can be solved by a truncation when $x$ is very large and by an approximation when $x$ is very close to 0. For more details see Appendices \ref{sec:truncate} and \ref{sec:heuristics}.

\section{Modeling Human Risk-sensitive Decision Making}
\label{sec:ex}

\subsection{Experiment}
Subjects were told that they are influential stock brokers, whose task is to invest into a fictive stock market (cf.\ \citealp{tobia2013}). At every trial (cf.~Fig.\ \ref{fig:mdp}a) subjects had to decide how much ($a=$ 0, 1, 2, or 3 EUR) to invest into a particular stock. After the investment, subjects first saw the change of the stock price and then were informed how much money they earned or lost. The received reward was proportional to the investment. The different trials, however, were not independent from each other (cf.~Fig.~\ref{fig:mdp}b). The sequential investment game consisted of 7 states, each one coming with a different set of contingencies, and subjects were transferred from one state to the next dependent of the amount of money they invested. For high investments, transitions followed the path labeled ``risk seeking'' (RS in Fig.~\ref{fig:mdp}b). For low investments, transitions followed the path labeled ``risk averse'' (RA in Fig.~\ref{fig:mdp}b). After 3 decisions subjects were always transferred back to the initial state, and the reward, which was accumulated during this round, was shown. State information was available to the subjects throughout every trial (cf.~Fig.~\ref{fig:mdp}a). Altogether, 30 subjects (young healthy adults) experienced 80 rounds of the 3-decision sequence.

Formally, the sequential investment game can be considered as an MDP with 7 states and 4 actions (see Fig.~\ref{fig:mdp}b). Depending on the strategy of the subjects, there are 4 possible paths, each of which is composed of 3 states. The total expected return for each path, averaged over all policies consistent with it, are shown in the right panels of Fig.~\ref{fig:mdp}b (``EV''). Path 1 provides the largest expected return per round (EV = 90), while Path 4 leads to an average loss of -9.75. Hence, to follow the on-average highest rewarded path 1, subjects have to take ``risky'' actions (investing 2 or 3 EUR at each state). Always taking conservative actions (investing 0 or 1 EUR) results in Path 4 and a high on-average loss. On the other hand, since the standard deviation of the return $R$ of each state equals $\textrm{std}(R) = a \times \textrm{C}$, where $a$ denotes the action (investment) the subject takes and $C$ denotes the price change, the higher the investment, the higher the risk. Path 1 has, therefore, the highest standard deviation (std = 14.9) of the total average reward, whereas the standard deviation of Path 4 is smallest (std = 6.9). Path 3 provides a trade-off option: it has slightly lower expected value (EV = 52.25) than Path 1 but comes with a lower risk (std = 12.3). Hence, the paradigm is suitable for observing and quantifying the risk-sensitive behavior of subjects.

\begin{figure}
 \centering
 \subfloat[Phase transition.]{\includegraphics[width=0.4\textwidth]{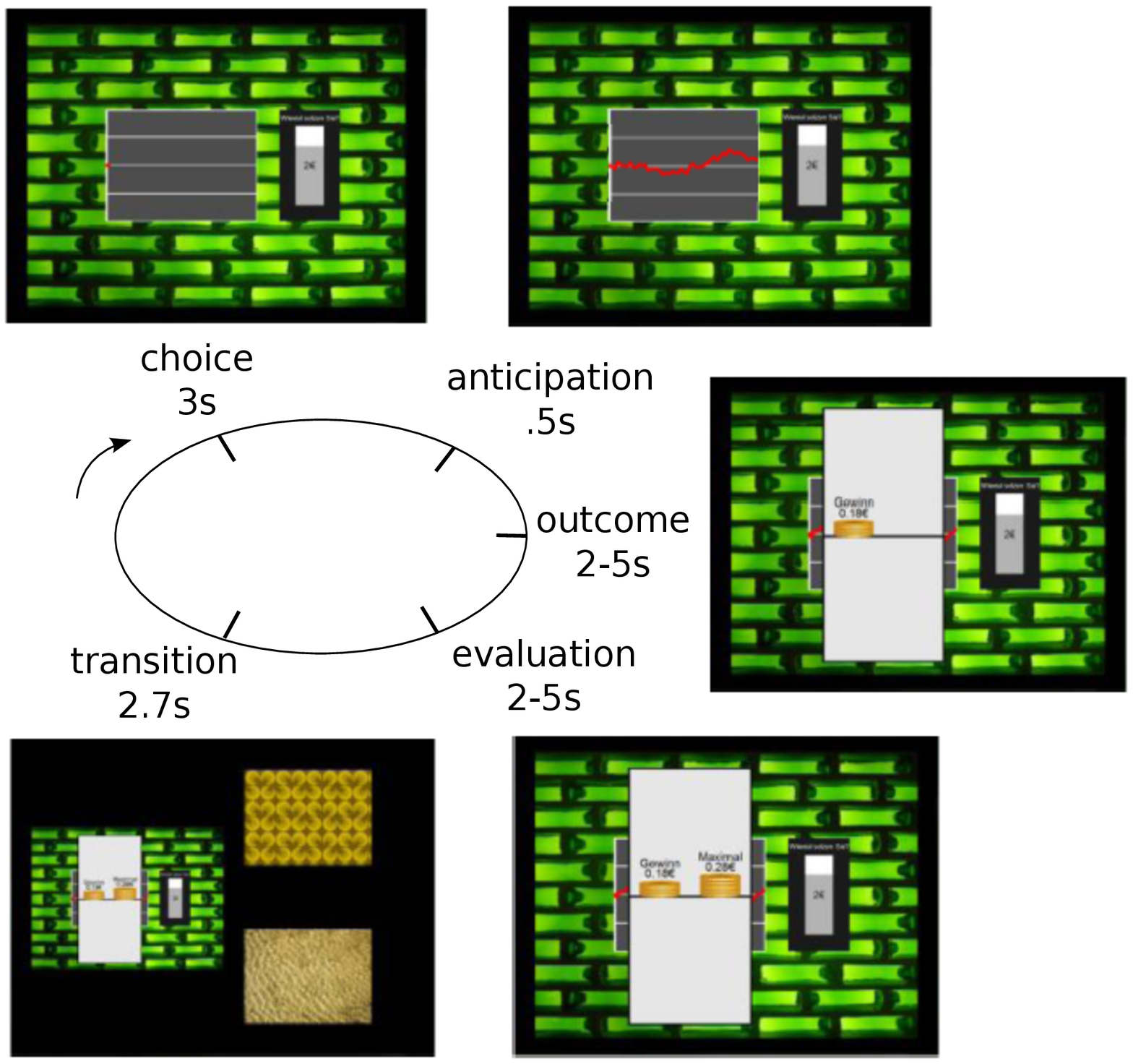}} \ 
 \subfloat[Structure of the underlying Markov decision process.]{\includegraphics[width=0.55\textwidth]{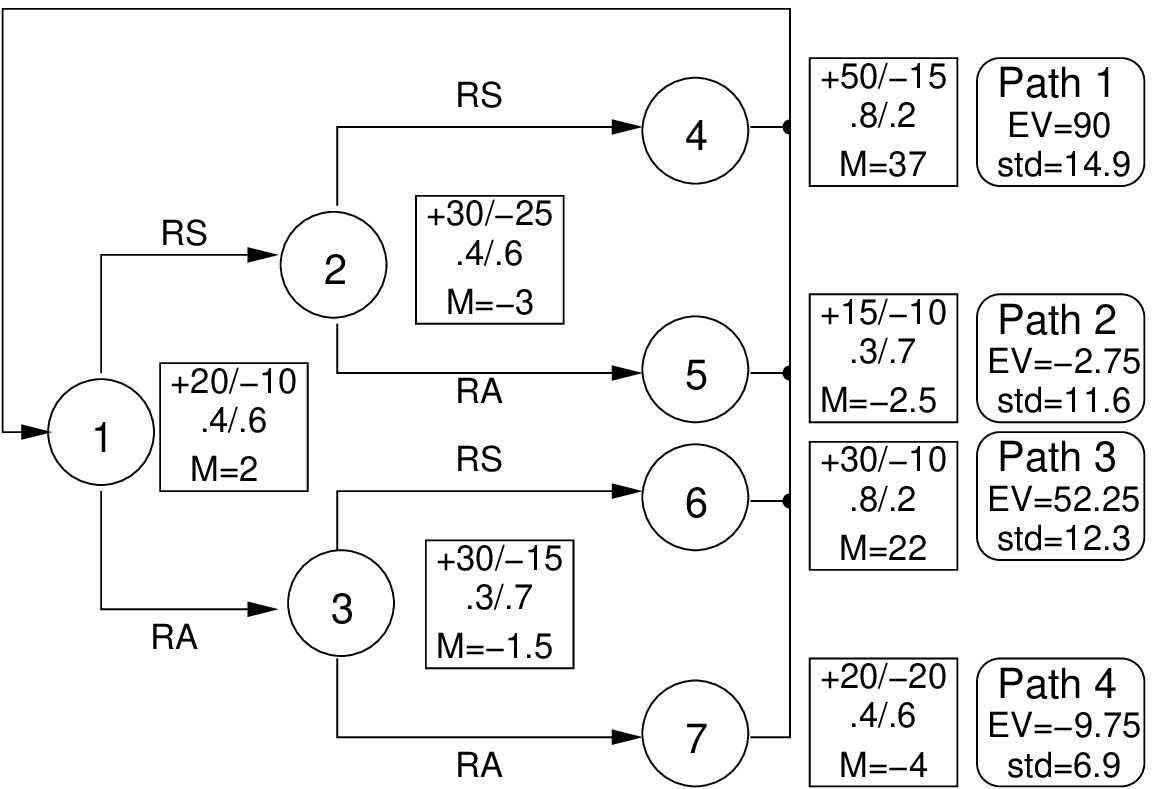}}
 \caption{The sequential investment paradigm. The paradigm is an implementation of a Markov decision process with 7 states and 4 possible actions (decisions to take) at every state. (a) Every decision (trial) consists of a choice phase (3s), during which an action (invest 0, 1, 2, or 3 EUR) must be taken by adjusting the scale bar on the screen, an anticipation phase (.5s), an outcome phase (2-5s), where the development of the stock price and the reward (wins and loses) are revealed, an evaluation phase (2-5s), where it reveals the maximal possible reward that could have been obtained for the (in hindsight) best possible action, and a transition phase (2.7s), where subjects are informed about the possible successor states and the specific transition, which will occur. The intervals of the outcome and evaluation phase are jittered for improved fMRI analysis. State information is provided by the colored patterns, the black field provides stock price information during anticipation phase, and the white field provides the reward and the maximal possible reward of this trial. After each round (3 trials), the total reward of this round is shown to subjects. (b) Structure of the underlying Markov decision process. The 7 states are indicated by numbered circles; arrows denote the possible transitions. Lables ``RS'' and ``RA'' indicate the transitions caused by the two ``risk-seeking'' (investment of 2 or 3 EUR) and the two ``risk-averse'' (investment of 0 or 1 EUR) actions. Bi-Gaussian distributions with a standard deviation of 5 are used to generate the random price changes of the stocks. Panels next to the states provide information about the means (top row) and the probabilities (center row) of ever component. M (bottom row) denotes the mean price change. The reward received equals the price change multiplied by the amount of money the subject invests. The rightmost panels provide the total expected rewards (EV) and the standard deviations (std) for all possible state sequences (Path 1 to Path 4) under the assumption that every sequence of actions consistent with a particular sequence of states is chosen with equal probability.}
 \label{fig:mdp}
\end{figure}

\begin{figure}[ht]
 \centering
   \includegraphics[width=0.8\textwidth]{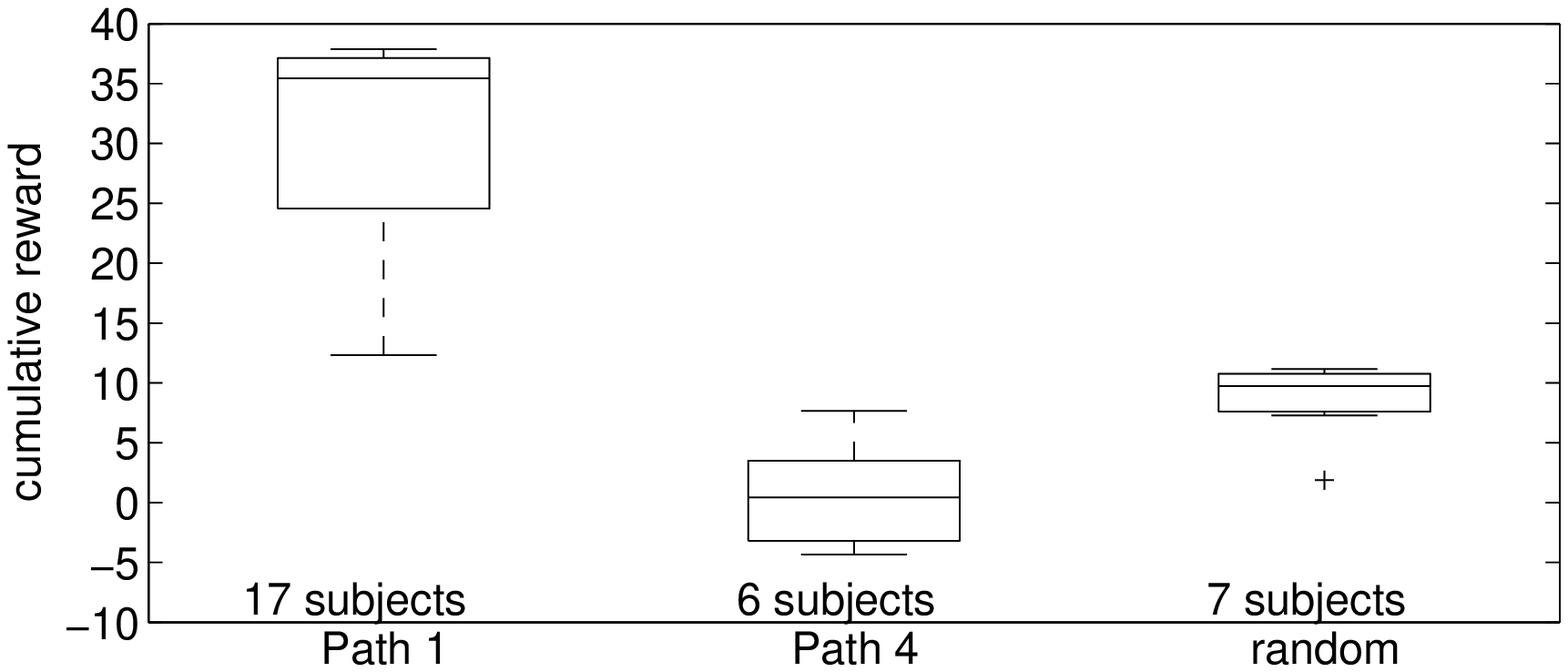}
\caption{Distribution of ``strategies'' chosen by the subjects in the sequential investment game and the corresponding cumulative rewards. Subjects are grouped according to the sequence of states (Path 1 to Path 4, cf.\ Fig.\ \ref{fig:mdp}b) they chose during the last 60 trials of the game. If a path $i$ is chosen in more than 60\% of the trials, the subject is assigned the group ``Path $i$''. Otherwise, subjects are assigned the group labeled ``random''. The vertical axis denotes the cumulative reward obtained during the last 60 trials.}
\label{fig:grouping}
\end{figure}

\subsection{Risk-sensitive Model of Human Behavior}
Fig.~\ref{fig:grouping} summarizes the strategies which were chosen by the 30 subjects. 17 subjects mainly chose Path 1, which provided them high rewards. 6 subjects chose Path 4, which gave very low rewards. The remaining 7 subjects show no significant preference among all 4 paths and the rewards they received are on average between the rewards received by the other 2 groups. The optimal policy for maximizing expected reward is the policy that follows Path 1. The results shown in Fig.~\ref{fig:grouping}, however, indicate that the standard model fails to explain the behavior of more than 40\% of the subjects.

We now quantify subjects' behavior by applying three classes of Q-learning algorithm: (1) standard Q-learning, (2) the risk-sensitive Q-learning (RSQL) method described by Algorithm \ref{alg:ql}, and (3) an expected utility (EU) algorithm with the following update rule
\begin{align}
 Q(s_t,a_t) \Leftarrow Q(s_t,a_t) + \alpha \left( u(r_t) - x_0 + \gamma \max_{a} Q(s_{t+1},a) - Q(s_t,a_t)\right), \label{eq:eu}
\end{align}
where the nonlinear transformation is applied to the reward $r_t$ directly. The latter one is a straightforward extension of expected utility theory. Risk-sensitivity is implemented via the nonlinear transformation of the true reward $r_t$. For both risk-sensitive Q-learning methods (RSQL and EU), we set the we set the reference level $x_0 = 0$ and consider the family of polynomial mixed utility functions
\begin{align}
 u(x) = \left\{
 \begin{array}{ll}
  k_+ x^{l_+} & x  \geq 0 \\
  - k_- (-x)^{l_-} & x < 0
 \end{array}
 \right..\label{eq:poly}
\end{align}
The parameters $k_{\pm} >0 $ and $l_\pm > 0$ quantify the risk-preferences separately for wins and losses 
(see Table \ref{tab:lpm}).
\begin{table}
\begin{center}\small
  \begin{tabular}{ | c | c | c | c }
    \hline
   branch $x\geq 0$  & shape & risk preference \\ \hline
   $0 <l_+ < 1$   & concave & risk-averse \\ \hline
    $l_+ = 1$ & linear & risk-neutral \\ \hline
    $l_+ > 1$ &  convex & risk-seeking \\ \hline
  \end{tabular}
  \ 
  \begin{tabular}{ | c | c | c | c }
    \hline
    branch $x < 0$ & shape & risk preference \\ \hline
     $0 < l_- < 1$ & convex & risk-seeking \\ \hline
    $l_- = 1$ & linear & risk-neutral \\ \hline
     $l_- > 1$  & concave & risk-averse \\ \hline
  \end{tabular}
\end{center}
\caption{Parameters for the two branches $x\geq 0$ (left) and $x<0$ (right) of the polynomial utility function $u(x)$ (Eq.~\eqref{eq:poly}), its shape and the induced risk preference.}
\label{tab:lpm}
\end{table}
Hence, there are 4 parameters for $u$ which have to be determined from the data. For all three classes, actions are generated according to the ``softmax'' policy Eq.\ \eqref{eq:softmax}, which is a proper policy for the paradigm (for proof see Appendix \ref{sec:softmax}), and the learning rate $\alpha$ is set constant across trials.

For RSQL, the learning rate is absorbed by the coefficients $k_\pm$. Hence, there are 6 parameters $\{ \beta, \gamma, k_{\pm}, l_{\pm}\} =: \theta$ which have to be determined. Standard Q-learning corresponds to the choice $l_{\pm} = 1$ and $k_\pm = \alpha$. The risk-sensitive model applied by \cite{niv2012neural} is also a special case of the RSQL-framework and corresponds $l_{\pm} = 1$. For the EU algorithm, there are 7 parameters, $\{ \alpha, \beta, \gamma, k_{\pm}, l_{\pm}\} =: \theta$, which have to be fitted to the data. $l_{\pm} = 1$ and $k_\pm = 1$ again corresponds to the standard Q-learning method.

Parameters were determined subject-wise by maximizing the log-likelihood of the subjects' action sequences,
\begin{align}
\max_{\theta} L(\theta) := \sum_{t=1}^{T} \log p(a_t|s_t, \theta) = \sum_{t=1}^{T} \log \frac{e^{\beta Q(s_t,a_t|\theta)}}{\sum_{a} e^{\beta Q(s_t,a|\theta)}} \label{eq:ml}
\end{align}
where $Q(s,a|\theta)$ indicates the dependence of the Q-values on the model parameters $\theta$. Since RSQL/EU and the standard Q-learning are nested model classes, we apply the Bayesian information criterion (BIC, see e.g.~\citealp{ghosh2006introduction})
\begin{align*}
 B := - 2 L + k \log(n)
\end{align*}
for model selection. $L$ denotes the log-likelihood, Eq.~\eqref{eq:ml}. $k$ and $n$ are the number of parameters and trials respectively.

To compare results, we report relative BIC scores, $\Delta B := B - B_Q$, where $B$ is the BIC score of the candidate model and $B_Q$ is the BIC score of the standard Q-learning model. We obtain
\begin{align*}
 \Delta B =& -500.14 \quad \textrm{ for RSQL, and } \\
 \Delta B =& -23.10  \quad \ \  \textrm{ for EU}.
\end{align*}
The more negative the relative BIC score is, the better the model fits data. Hence, the RSQL algorithm provides a significantly better explanation for the behavioral data than the EU algorithm and standard Q-learning. In the following, we only discuss the results obtained with the RSQL model.

\begin{figure}[ht]
 \centering
   \includegraphics[width=0.8\textwidth]{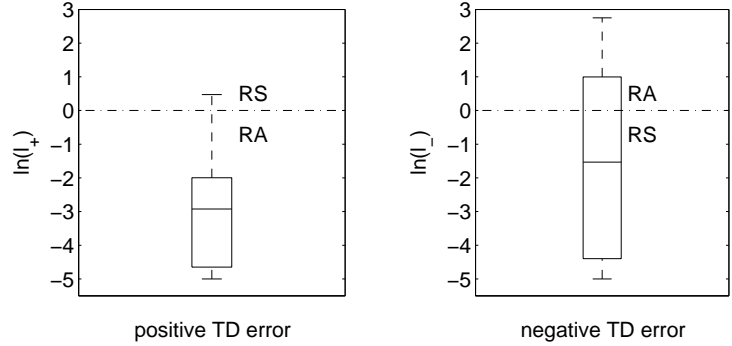}
\caption{Distribution of values for the shape parameters $l_+$ (left) and $l_-$ (right) for the RSQL model.}
\label{fig:para}
\end{figure}

Fig.\ \ref{fig:para} shows the distribution of best-fitting values for the two parameters $l_{\pm}$ which quantify the risk-preferences of the individual subjects. We conclude (cf.~Table \ref{tab:lpm}) that most of the subjects are risk-averse for positive and risk-seeking for negative TD errors. The result is consistent with previous studies from the economics literature (see \citealp{tversky1992advances}, and references therein). 

After determining the parameters $\{k_\pm, l_\pm \}$ for the utility functions, we perform an analysis similar to the analysis discussed in Section \ref{sec:toyex}. Given an observed reward sequence $\{ r_i\}_{i=1}^N$, the empirical subjective mean $m_{sub}$ is obtained by solving the following equation
\begin{align*}
 \frac{1}{N}\sum_{i=1}^N u(r_i - m_{sub}) = 0.
\end{align*}
If subjects are risk-neutral, then $u(x) = x$, and $m_{sub} = m_{emp} = \frac{1}{N} \sum_{i=1}^N r_i$ is simply the empirical mean. Following the idea of prospect theory, we define a normalized subjective probability 
$\Delta p$,
\begin{align}
 \Delta p := \frac{m_{sub} - \min_{i} r_{i}}{\max_{i} r_{i} - \min_{i} r_{i}} - \frac{m_{emp} -  \min_{i} r_{i}}{\max_{i} r_{i} - \min_{i} r_{i}} = \frac{m_{sub} - m_{emp}}{\max_{i} r_{i} - \min_{i} r_{i}}. \label{eq:relP}
\end{align}
If $\Delta p$ is positive, the probability of rewards is overestimated and the induced policy is, therefore, risk-seeking. If $\Delta p$ is negative, the probability of rewards is underestimated and the policy is risk-averse. Fig.\ \ref{fig:subprob} summarizes the distribution of normalized subjective probabilities for subjects from the ``Path 1'', ``Path 4'' and ``random'' groups of Fig.~\ref{fig:grouping}.  For subjects within group ``Path 1'', $\lvert \Delta p \rvert$ is small and their behaviors are similar to those of risk-neutral agents. This is consistent with their policy, because both risk-seeking and risk-neutral agents should prefer Path 1. For subjects within groups ``Path 4'' and ``random'', the normalized subjective probabilities are on average 10 \% lower than those of risk-neutral agents. This explains why subjects in these groups adopt the conservative policies and only infrequently choose Path 1. 

\begin{figure}[ht]
 \centering
   \includegraphics[width=0.8\textwidth]{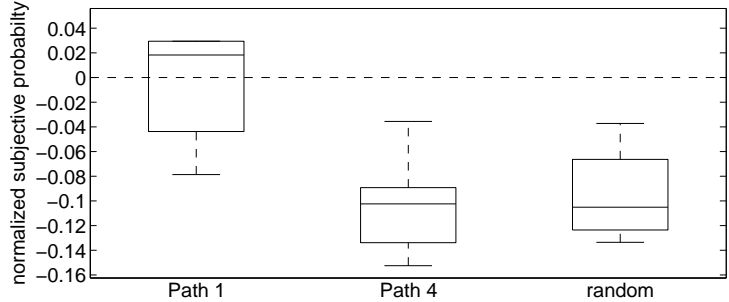}
\caption{Distribution of normalized subjective probabilities, $\Delta p$, Eq.~\eqref{eq:relP}, for the different subject groups defined in Fig.~\ref{fig:grouping}. }
\label{fig:subprob}
\end{figure}

\subsection{fMRI Results}
Functional magnetic resonance imaging (fMRI) data were simultaneously recorded while subjects played the sequential investment game. The analysis of fMRI data was conducted in SPM8 (Wellcome Department of Cognitive Neurology, London, UK; details of the magnetic resonance protocol and data processing are presented in Appendix \ref{sec:mr}). The sequence of Q-values for the action chosen at each state were used as parametric modulators during the choice phase, and temporal difference (TD) errors were used at the outcome phase (see Fig.~\ref{fig:mdp}a).

\begin{figure}
 \centering
 \subfloat[TD errors]{\includegraphics[width=0.45\textwidth]{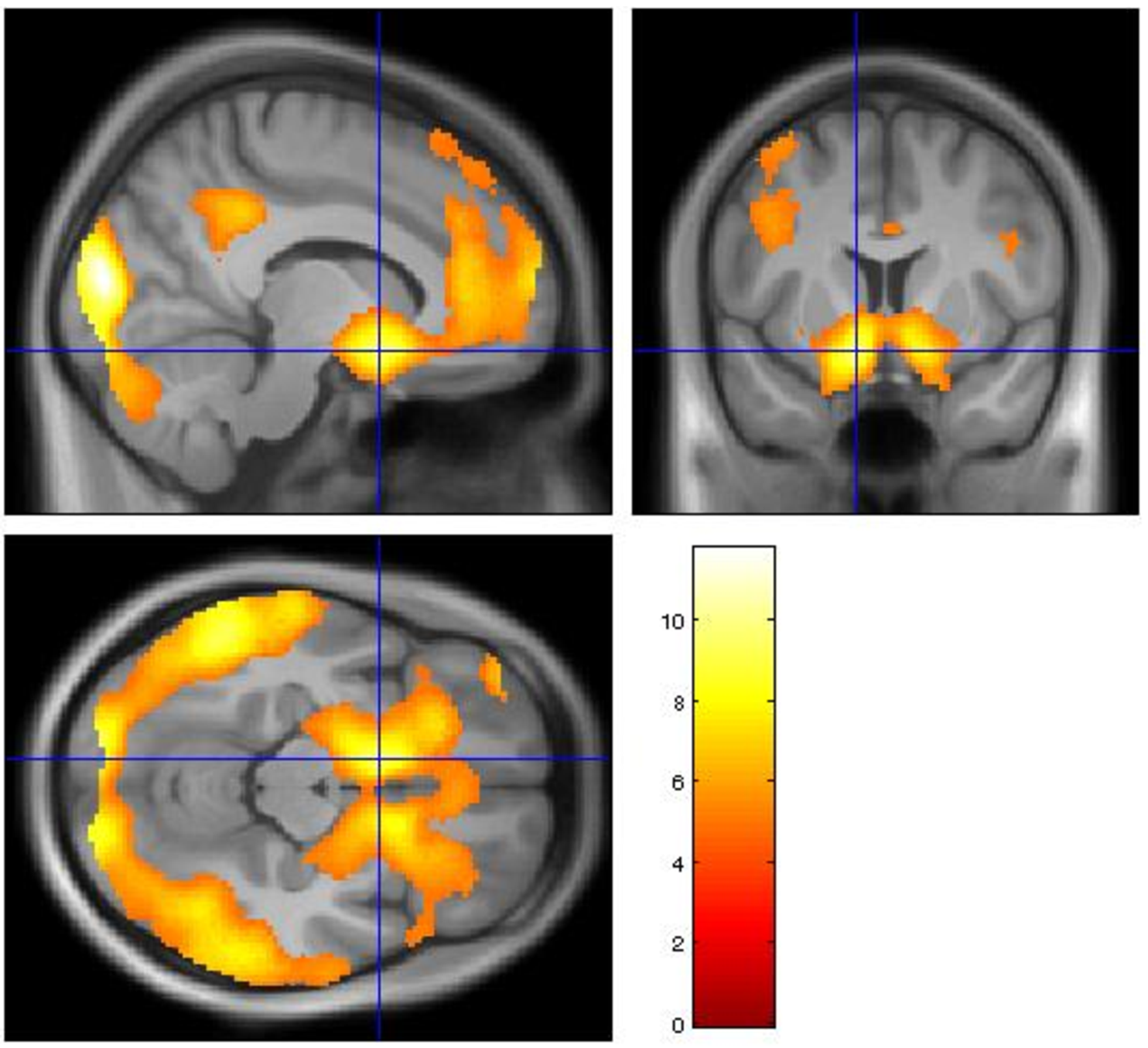}} \label{fig:fmriTD}\ 
 \subfloat[Q-values]{\includegraphics[width=0.45\textwidth]{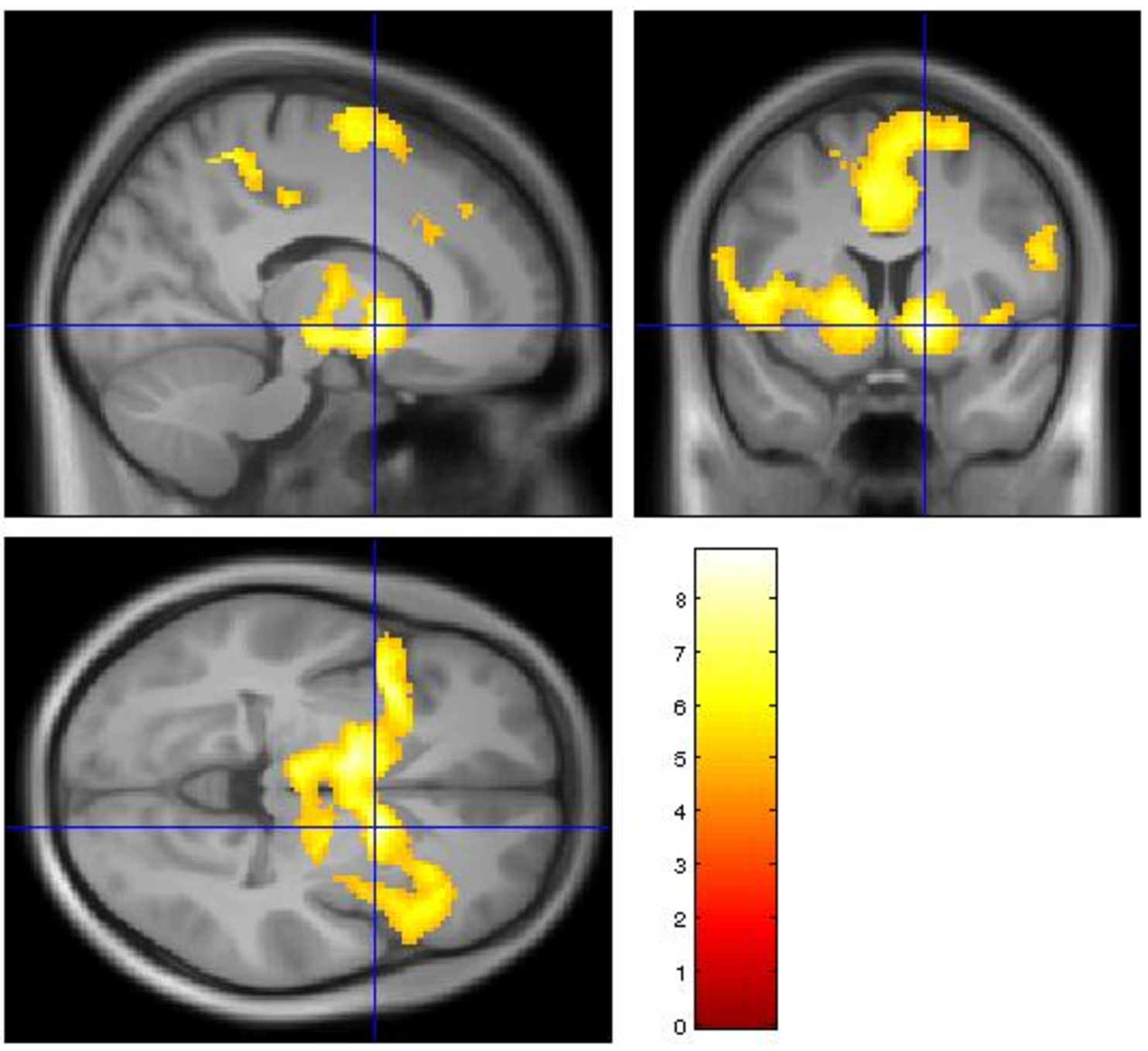}} \label{fig:fmriQ}
 \caption{Modulation of the fMRI BOLD signal by TD errors (a) and by Q-values (b) generated by the RSQL model with best fitting parameters. The data is shown whole-brain corrected to $p<.05$ (voxel-wise $p<.001$ and minimum 125 voxels). The color bar indicates the $t$-value ranging from 0 to the maximal value. The cross indicates location of strongest modulation for TD errors (in the left ventral striatum (-14 8 -16)) and for Q-values (in the right ventral striatum (14 8 -4)). However, it is remarkable that for both TD errors and Q-values, modulations in the left and right ventral striatum are almost equally strong with a slight difference.}
 \label{fig:fmri}
\end{figure}

Fig.~\ref{fig:fmri}a shows that the sequence of TD errors for the RSQL model (with best fitting parameters) positively modulated the BOLD signal in the subcallosal gyrus extending into the ventral striatum  (-14 8 -16) (marked by the cross in Fig.~\ref{fig:fmri}a), the anterior cingulate cortex (8 48 6), and the visual cortex (-8 -92 16; $z=7.9$). The modulation of the BOLD signal in the ventral striatum is consistent with previous experimental findings (cf.\ \citealp{schultz2002getting,o2004reward}), and supports the primary assertion of computational models that reward-based learning occurs when expectations (here, expectations of ``subjective'' quantities) are violated \citep{sutton1998reinforcement}.

Fig.~\ref{fig:fmri}b shows the results for the sequence of Q-values for the RSQL model (with best fitting parameters), which correspond to the subjective (risk-sensitive) expected value of the reward for each discrete choice. Several large clusters of voxels in cortical and subcortical structures were significantly modulated by the Q-values at the moment of choice. The sign of this modulation was negative. 
The peak of this negative modulation occurred in the left anterior insula (-36 12 -2, $z=4.6$ ), with strong modulation also in the
bilateral ventral striatum (14 8 -4, marked by the cross in Fig.~\ref{fig:fmri}b; -16 4 0) and the cingulate cortex (4 16 28). The reward prediction error processed by the ventral striatum (and other regions noted above) would not be computable in the absence of an expectation, and as such, this activation is important for substantiating the plausibility for the RSQL model of learning and choice. Sequences of Q-values obtained with standard Q-learning (with best fitting parameters), on the other hand, failed to predict any changes in brain activity even at a liberal statistical threshold of $p < .01$ (uncorrected). This lack of neural activity for the standard Q model, in combination with the significant activation for our RSQL, supports the hypothesis that some assessment of risk is induced and influences valuation. Whereas the areas modulated by Q-values differ from what has been reported in other studies (i.e., the ventromedial prefrontal cortex as in \citealp{glascher2009determining}), it does overlap with the representation of TD errors. Furthermore, the opposing signs of the correlated neural activity suggests that a neural mismatch of signals in the ventral striatum between Q-value and TD errors may underlie the mechanism by which values are learned.


\subsection{Discussion}
We applied the risk-sensitive Q-learning (RSQL) method to quantify human behavior in a sequential investment game and investigated the correlation of the predicted TD- and Q-values with the neural signals measured by fMRI.

We first showed that the standard Q-learning algorithm cannot explain the behavior of a large number of subjects in the task. Applying RSQL generated a significantly better fit and also outperformed the expected utility algorithm. The risk-sensitivity revealed by the best fitting parameters is consistent with the studies in behavioral economics, that is, subjects are risk-averse for positive while risk-seeking for negative TD errors. Finally, the relative subjective probabilities provide a good explanation why some subjects take conservative policies: they underestimate the true probabilities of gaining rewards. 

The fMRI results showed that TD sequence generated by our model has a significant correlation with the activity in the subcallosal gyrus extending into the ventral striatum. The sequence of Q-values has a significant correlation with the activities in the left anterior insula. 
Previous studies (see e.g., Chapter 23 of \citealp{glimcher2008neuroeconomics} and \citealp{symmonds2011deconstructing}) suggest that higher order statistics of outcomes, e.g., variance (the 2nd order) and skewness (the 3rd order), are encoded in human brains separately and the individual integration of these risk metrics induces the corresponding risk-sensitivity. Our results indicate, however, that the risk-sensitivity can be simply induced (and therefore encoded) by a nonlinear transformation of TD errors and no additional neural representation of higher order statistics is needed. 

\section{Summary}
We applied a family of valuation functions, the utility-based shortfall, to the general framework of risk-sensitive Markov decision processes, and we derived a risk-sensitive Q-learning algorithm. We proved that the proposed algorithm converges to the optimal policy corresponding to the risk-sensitive objective. By applying S-shape utility functions, we show that key features predicted by prospect theory can be replicated using the proposed algorithm. Hence, the novel Q-learning algorithm provides a good candidate model for human risk-sensitive sequential decision-making procedures in learning tasks, where mixed risk-preferences are shown in behavioral studies. We applied the algorithm to model human behaviors in a sequential investment game. The results showed that the new algorithm fitted data significantly better than the standard Q-learning and the expected utility model. 
The analysis of fMRI data shows a significant correlation of the risk-sensitive TD error with the BOLD signal change in the ventral striatum, and also a significant correlation of the risk-sensitive Q-values with neural activity in the striatum, cingulate cortex and insula, which is not present if standard Q-values are applied.

Some technical extensions are possible within our general risk-sensitive reinforcement learning (RL) framework: (a) The Q-learning algorithm derived in this paper can be regarded a special type of RL algorithms, TD(0). It can be extended to other types of RL algorithms like SARSA and TD($\lambda$) for $\lambda \neq 0$. (b) In our previous work \citep{Shen2013}, we also provided a framework for the average case. Hence, RL algorithms for the average case can also be derived similar to the discounted case considered in this paper. (c) The algorithm in its current form applies to MDPs with finite state spaces only. It can be extended for MDPs with  continuous state spaces by applying function approximation technique.

\subsection*{Acknowledgments}
Thanks to Wendelin B\"ohmer, Rong Guo and Maziar Hashemi-Nezhad for useful discussions and suggestions and to the anonymous referee for helpful comments. The work of Y.\ Shen and K.\ Obermayer was supported by the BMBF (Bersteinfokus Lernen TP1), 01GQ0911, and the work of M.J.\ Tobia and T.\ Sommer was supported by the BMBF (Bernsteinfokus Lernen TP2), 01GQ0912.

\appendix

\section{Mathematical Proofs}

The sup-norm is defined as $\lVert X \rVert_\infty := \max_{i \in I} \lvert X(i) \rvert,$
where $X = [X(i)]_{i \in I}$ can be considered as a $\lvert I \rvert$-dimensional vector. 

\begin{lemma}\label{lm:inrm}
Let $\rho$ be valuation function on $\mathbb R^{\lvert I \rvert} \times \mathscr P$ and $\tilde \rho(X, \mu)  := \rho(X, \mu) -\rho(\mathbf 0, \mu)$. Then the following inequality holds
\begin{align*}
 \min_{i \in I} X_i =: \underline{X} \leq \tilde \rho(X, \mu) \leq \overline{X}:= \max_{i \in I} X_i, \forall \mu \in \mathscr P, X \in \mathbb R^{\lvert I \rvert}.
\end{align*}
\end{lemma}
\begin{proof}
 By $\underline{X} \leq X_i \leq \overline{X}, \forall i \in I$ and monotonicity of valuation functions, we obtain
 \begin{align*}
  \rho(\underline{X} \mathbf 1, \mu) \leq \rho(X, \mu) \leq \rho(\overline{X} \mathbf 1, \mu).
 \end{align*}
Due to the translation invariance, we have then
\begin{align*}
 \rho(\overline{X} \mathbf 1, \mu) = \rho(\mathbf 0, \mu) +  \overline{X}, \textrm{ and } \rho(\underline{X} \mathbf 1, \mu) = \rho(\mathbf 0, \mu) +  \underline{X}.
\end{align*}
which immediately imply that
\begin{align*}
 \overline{X} \leq \rho(X, \mu) - \rho(\mathbf 0, \mu) \leq \overline{X}, \forall \mu \in \mathscr P, X \in \mathbb R^{\lvert I \rvert}.
\end{align*}
\end{proof}

\begin{proof}[Proof of Proposition \ref{prop:impl}]
 (ii) $\Rightarrow$ (i). By definition, $m^* \leq \rho_{x_0}^{\textrm{u}}(X)$. For any $\epsilon > 0$, since $u$ is strictly increasing, we have $u(X(i) - m^* - \epsilon) < u(X(\omega) - m^*), \forall i \in I$, which implies $\mathbb E u(X - m^* - \epsilon ) < \mathbb E u(X - m^*) = x_0$. Hence, $m^* = \rho_{x_0}^{\textrm{u}}(X)$. 

(i) $\Rightarrow$ (ii). By definition we have $\mathbb E u(X - m^*) \geq x_0$. Assume that $\mathbb E u(X - m^*) > x_0$. By the continuity of $u$, there exists an $\epsilon > 0$ such that $\mathbb E u(X - m^* - \epsilon) > x_0$, which implies $\rho_{x_0}^{\textrm{u}}(X) \geq m^* + \epsilon > m^*$ and hence contradicts (i). Thus, (ii) holds.
\end{proof}

\subsection{Proofs for Risk-sensitive Q-learning}
\label{sec:proof_rl}
Before proving the risk-sensitive Q-learning, we consider a more general update rule
\begin{align}
 q_{t+1}(i) = (1-\alpha_t(i)) q_t(i) + \alpha_t(i) \left[ (H q_t)(i) + w_t(i) \right]. \label{eq:qiter}
\end{align}
where $q_t \in \mathbb R^d$, $H: \mathbb R^d \rightarrow \mathbb R^d$ is an operator, $w_t$ denotes some random noise term and $\alpha_t$ is learning rate with the understanding that $\alpha_t(i) = 0$ if $q(i)$ is not updated at time $t$. Denote by $\mathcal F_t$ the history of the algorithm up to time $t$,
\begin{align*}
 \mathcal F_t = \{ q_0(i), \ldots, q_t(i), w_0(i), \ldots, w_{t-1}(i), \alpha_0(i), \ldots, \alpha_t(i), i=1,\ldots,t \}.
\end{align*}
We restate the following proposition. 
\begin{proposition}[Proposition 4.4, \citealp{bertsekas1996neuro}] 
\label{prop:sapp}
 Let $q_t$ be the sequence generated by the iteration \eqref{eq:qiter}. We assume the following
\begin{itemize}
 \item[a] The learning rates $\alpha_t(i)$ are nonnegative and satisfy
\begin{align*}
 \sum_{t=0}^\infty \alpha_t(i) = \infty, \quad \sum_{t=0}^\infty \alpha_t^2(i) = \infty, \forall i
\end{align*}
 \item[b] The noise terms $w_t(i)$ satisfy (i) for every $i$ and $t$, $\mathbb E[w_t(i)|\mathcal F_t] = 0$; (ii) Given some norm $\lVert \cdot \rVert$ on $\mathbb R^d$, there exist constants $A$ and $B$ such that $\mathbb E[w_t^2(i)|\mathcal F_t] \leq A + B \rVert q_t \rVert^2$. 
\item[c] The mapping $H$ is a contraction under sup-norm.
\end{itemize}
Then $q_t$ converges to the unique solution $q^*$ of the equation $Hq^* = q^*$ with probability 1. 
\end{proposition}

To apply Proposition \ref{prop:sapp}, we first reformulate the Q-learning rule \eqref{eq:ql} in a different form
\begin{align*}
 q_{t+1}(s,a) = (1-\frac{\alpha_t(s,a)}{\alpha}) q_t(s,a) + \frac{\alpha_t(s,a)}{\alpha} \left[ \alpha u(d_t) - x_0 + q_t(s,a) \right]
\end{align*}
where $\alpha$ denotes an arbitrary constant such that $\alpha \in (0,\min(L^{-1},1)]$. Recall that $L$ is defined in Assumption \ref{ass:2}. For simplicity, we define $\tilde u(x) := u(x) - x_0$, $d_t := r_t + \gamma \max_{a} q_t(s_{t+1},a) - q_t(s,a)$ and set
\begin{align}
 (Hq_t)(s,a) =& \alpha \mathbb E_{s,a} \tilde u(r_t + \gamma \max_{a} q_t(s_{t+1},a) - q_t(s,a)) + q_t(s,a) \\
 w_t(s,a) = & \alpha \tilde u(d_t) - \alpha \mathbb E_{s,a} \tilde u(r_t + \gamma \max_{a} q_t(s_{t+1},a) - q_t(s,a)) \label{eq:w}
\end{align}
More explicitly, $Hq$ is defined as
\begin{align*}
 (Hq)(s,a) = \alpha \sum_{s', \varepsilon} \tilde{\mathcal P}(s',\epsilon|s,a) \tilde u\left(r(s,a,\varepsilon) + \gamma \max_{a'} q(s',a') - q(s,a)\right) + q(s,a),
\end{align*}
where $\tilde{\mathcal P}(s',\epsilon|s,a) := \mathcal P(s'|s,a) \mathcal P_r(\varepsilon|s,a)$. We assume the size of the space $\mathbf K$ is $d$.
\begin{lemma}\label{lm:cont}
 Suppose that Assumption \ref{ass:2} holds and $0 < \alpha \leq \min(L^{-1},1)$. Then there exists a real number $\bar \alpha \in [0,1)$ such that for all $q,q' \in \mathbb R^d$, $\lVert Hq - H q' \rVert_{\infty} \leq \bar \alpha \lVert q - q'\rVert_\infty$.
\end{lemma}
\begin{proof}
Define $v(s):= \max_{a} q(s,a)$ and $v'(s):= \max_{a} q'(s,a)$. Thus, $$\rvert v(s) - v(s) \lvert \leq \max_{(s,a)\in \mathbf K} \rvert q(s,a) - q'(s,a) \lvert = \lVert q - q'\rVert_\infty.$$ By Assumption \ref{ass:2} (ii) and the monotonicity of $\tilde u$, there exists a $\xi_{(x,y)} \in [\epsilon,L]$ such that $\tilde u(x) - \tilde u(y) = \xi_{(x,y)}(x-y)$. Analogously, we obtain
\begin{align*}
  & (Hq)(s,a) - (Hq')(s,a) \\
= &\sum_{s',\varepsilon} \tilde{\mathcal P}(s',\epsilon|s,a) \{ \alpha \xi_{(s,a,\varepsilon,s',q,q')} [\gamma v(s') - \gamma v'(s') - q(s,a) + q'(s,a)]  \\
 &  + (q(s,a) - q'(s,a)) \}\\
 = & \alpha \gamma \sum_{s', \varepsilon} \tilde{\mathcal P}(s',\epsilon|s,a) \xi_{(s,a,\varepsilon,s',q,q')} [v(s') - v'(s')] \\
 & + (1- \alpha \sum_{s', \varepsilon} \tilde{\mathcal P}(s',\epsilon|s,a) \xi_{(s,a,\varepsilon,s',q,q')}) [q(s,a)-q'(s,a)] \\
 \leq & \left(1 - \alpha(1 - \gamma) \sum_{s', \varepsilon} \tilde{\mathcal P}(s',\epsilon|s,a) \xi_{(s,a,\varepsilon,s',q,q')} \right) \lVert q - q'\rVert_\infty \\
 \leq & \left( 1 - \alpha(1 - \gamma) \epsilon \right) \lVert q - q'\rVert_\infty
\end{align*}
Hence, $\bar \alpha = 1 - \alpha (1 - \gamma) \epsilon$ is the required constant.
\end{proof}

\begin{proof}[Proof of Theorem \ref{th:ql}] Obviously, Condition (a) in Proposition \ref{prop:sapp} is satisfied and Condition (c) holds also due to Lemma \ref{lm:cont}. It remains to check Condition (b).

$\mathbb E [w_t(s,a) | \mathcal F_t] = 0$ holds by its definition in \eqref{eq:w}. Next we prove (ii). In fact, 
$$\mathbb E [w_t^2(s,a) | \mathcal F_t] = \alpha^2 \mathbb E \left[(\tilde u(d_t))^2 | \mathcal F_t\right] - \alpha^2 (\mathbb E \left[\tilde u(d_t) | \mathcal F_t\right])^2 \leq \alpha^2 \mathbb E \left[(\tilde u(d_t))^2 | \mathcal F_t\right]$$
Let $\bar R$ be the upper bound for $r_t$. Then $ \lvert d_t \rvert \leq \bar R + 2 \lVert q_t \rVert_\infty$, which implies that $\lvert \tilde u(d_t) -\tilde u(0) \rvert \leq L ( \bar R + 2 \lVert q_t \rVert_\infty)$ due to Assumption \ref{ass:2}(ii). Hence, $\lvert  \tilde u(d_t) \rvert \leq \lvert \tilde u(0) \rvert + L ( \bar R + 2 \lVert q_t \rVert_\infty)$. On the other hand, since
\begin{align*}
 (\lvert \tilde u(0) \rvert + L \bar R + 2 L \lVert q_t \rVert_\infty)^2 \leq 2 (\lvert \tilde u(0) \rvert + L \bar R)^2 + 8L^2 \lVert q_t \rVert_\infty^2
\end{align*}
we have $\alpha^2 \mathbb E \left[(\tilde u(d_t))^2 | \mathcal F_t\right] \leq 2 \alpha^2 (\lvert \tilde u(0) \rvert + L \bar R)^2 + 8 \alpha^2 L^2 \lVert q_t \rVert_\infty^2$. Hence, Condition (b) holds. 
\end{proof}

\subsection{Truncated Algorithms with Weaker Assumptions}
Some functions like $u(x) = e^x$ and $u(x) = x^p$, $p >0$, do not satisfy the global Lipschitz condition required in Assumption \ref{ass:2} (ii). In real applications, however, we can relax the assumption to assume that the Lipschitz condition holds locally within a ``sufficiently large'' subset. Lemma \ref{lm:bound} states such subset provided the upper bound of absolute value of rewards is known.

\begin{assumption}\label{ass:1}
 The reward function $r(s,a,\epsilon)$ is bounded under sup-norm, i.e., $$\bar R := \sup_{(s,a) \in \mathbf K, \epsilon \in \mathbf E} \lvert r(s,a,\epsilon) \rvert < \infty.$$
\end{assumption}

\label{sec:truncate}
Define an operator $\mathcal T: \mathbb R^{\lvert \mathbf S \rvert} \rightarrow \mathbb R^{\lvert \mathbf S \rvert}$ as
\begin{align*}
 \mathcal T_s(V) = \max_{a \in \mathbf A(s)} \mathcal U_{s,a}(R(s,a) + \gamma V).
\end{align*}

\begin{lemma}[cf.~Lemma 5.4, \citealp{Shen2013}]
$\mathcal T$ is a contracting map under sup-norm, i.e., $$\lVert \mathcal T(V) - \mathcal T(V') \rVert_\infty \leq \gamma \lVert V - V' \rVert_\infty, \forall V,V' \in \mathbb R^{\lvert \mathbf S \rvert}.$$ \label{lm:con}
\end{lemma}

\begin{lemma}\label{lm:bound}
 Under Assumption \ref{ass:2} (i) and \ref{ass:1}, applying the valuation map in \eqref{eq:opt}, the solution $Q^*$ satisfies $ \frac{-\bar R - y_0}{1-\gamma} \leq Q^*(s,a) \leq \frac{\bar R - y_0}{1-\gamma}, \forall (s,a) \in \mathbf K.$
\end{lemma}
\begin{proof}
By assumption, $u^{-1}(x_0)$ exists. Since $u$ is strictly increasing, we have $\mathcal U_{s,a}(0) = \sup \{ m \in \mathbb{R} | u(-m) \geq x_0 \} = - u^{-1}(x_0)$. Hence, together with Eq.~\eqref{eq:range}, we obtain for all $(s,a) \in \mathbf K$,
\begin{align*}
 - u^{-1}(x_0) - \bar R = \mathcal U_{s,a}(0) - \bar R \leq \mathcal U_{s,a}(R) \leq \mathcal U_{s,a}(0) + \bar R = - u^{-1}(x_0)+ \bar R 
\end{align*}
Note that Lemma \ref{lm:con} implies that $V^* = \mathcal T^\infty(V_0)$ for any $V_0 \in \mathbb R^{\lvert \mathbf S \rvert}$. Without loss of generality, we start from $V_0 = 0$. Define $\underline u := - u^{-1}(x_0) - \bar R$ and $\bar u := - u^{-1}(x_0) + \bar R$. Hence, we have $\underline u \leq \mathcal T(0) = \max_{a} \mathcal U_{s,a}(R) \leq \bar u$, which implies 
\begin{align*}
 &\mathcal T^2(0) = \max_{a} \mathcal U_{s,a}(R + \gamma \mathcal T(0)) \leq \max_{a} \mathcal U_{s,a}(R ) + \gamma \bar u \leq (1+\gamma) \bar u \\
\textrm{and} \quad & \mathcal T^2(0) = \max_{a} \mathcal U_{s,a}(R + \gamma \mathcal T(0)) \geq \max_{a} \mathcal U_{s,a}(R ) + \gamma \underline u \geq (1+\gamma) \underline u& 
\end{align*}
Repeating above procedure, we obtain $(1+\gamma +\ldots + \gamma^{n-1}) \underline u \leq \mathcal T^n(0) \leq (1+\gamma +\ldots + \gamma^{n-1}) \bar u$. Hence, $\frac{\underline u}{1 - \gamma} \leq V^* = \mathcal T^\infty(0) \leq \frac{\bar u}{1 - \gamma}$. 
By the definition of $Q^*$, above inequalities hold for $Q^*$ as well. 
\end{proof}
Define
\begin{align}
 \underline x := y_0 - \frac{2 \bar R}{1-\gamma} \quad \textrm{and} \quad \bar x:= y_0 + \frac{2 \bar R }{1-\gamma} \label{eq:trun}
\end{align}
Given Lemma \ref{lm:bound}, we can truncate the utility function $u$ outside the interval $[\underline x, \overline x]$ as
\begin{align}
 u'(x) = \left\{ \begin{array}{ll}
                 u(\underline x) + \epsilon(x - \underline x), & x \in (-\infty, \underline x)\\
                 u(x), & x \in [\underline{x}, \bar x]\\
                 u(\bar x) + \epsilon(x - \bar x), & x \in (\bar x, \infty)
                \end{array}
\right.. \label{eq:truncate}
\end{align}

\begin{theorem}\label{th:truncate}
 Suppose that Assumption \ref{ass:2} (i) and \ref{ass:1} hold. Assume further that There exist positive constants $\epsilon, L\in \mathbb R^+$ such that $ 0 < \epsilon \leq \frac{u(x) - u(y)}{x-y} \leq L $, for all $x \neq y \in [\underline x, \bar x]$, where $\underline x, \overline x$ are defined in Eq.~\eqref{eq:trun}. Then the unique solution $Q^*_1$ to Eq.~\eqref{eq:opt} with $u$ and the unique solution $Q^*_2$ to Eq.~\eqref{eq:opt} with $u'$ are identical.
\end{theorem}
\begin{proof}
 Both uniqueness is due to Theorem \ref{th:vi} and Proposition \ref{prop:impl}. By Lemma \ref{lm:bound}, $\frac{-\bar R - y_0}{1-\gamma} \leq Q^*_i(s,a) \leq \frac{\bar R - y_0}{1-\gamma}$ hold for all $(s,a) \in \mathbf K$ and $i = 1, 2$. Hence, we have for both $i = 1, 2$ and for all $(s,a), (s',a') \in \mathbf K, \epsilon \in \mathbf E$, 
 \begin{align*}
  y_0 - \frac{2 \bar R}{1-\gamma} \leq r(s,a.\epsilon) + \gamma Q^*_i(s',a') - Q^*_i(s,a) \leq y_0 + \frac{2 \bar R }{1-\gamma}.
 \end{align*}
Since $u$ and $u'$ are identical within the set $[\underline x, \bar x]$, $Q^*_1(s,a) = Q^*_2(s,a)$ for all $(s,a) \in \mathbf K$.
\end{proof}

Now we state the risk-sensitive Q-learing algorithm with truncation.
\begin{algorithm}[ht]                      
\caption{Q-learning with truncation}          
\label{alg:ql2}                          
\begin{algorithmic}
\State initialize $Q(s,a) = 0$ and $N(s,a) = 0$ for all $s,a$.
\For{$t=1$ to $T$}
 \State at state $s_t$ choose action $a_t$ randomly using a proper policy (e.g.\ Eq.~\eqref{eq:softmax});
\State observe date $(s_t,a_t,r_t, s_{t+1})$;
 \State $N(s_t,a_t) \Leftarrow N(s_t,a_t) + 1$ and set learning rate: $\alpha_t := 1/N(s_t,a_t)$;
 \State update $Q$ as in Eq.~\eqref{eq:ql};
 \State truncate $Q$ as in Eq.~\eqref{eq:truncate},  
 where $\bar x$ and $\underline x$ are defined in Eq.~\eqref{eq:trun}.
\EndFor
\end{algorithmic}
\end{algorithm}

\subsection{Heuristics for Polynomial Utility Functions} \label{sec:heuristics}
So far we have relaxed the assumption for utility functions to locally Lipschitz. However, some functions of interest are even not locally Lipschitz. For instance, the function $u(x) = x^p$, $p \in (0,1)$ is not Lipschitz at the area close to 0. We suggest two types of approximation to avoid this problem.
\begin{enumerate}
 \item Approximate $u$ by $u^\varphi(x) = (x + \varphi)^p - \varphi^p$ with some positive $\varphi$.
 \item Approximate $u$ close to 0 by a linear function, i.e.
\begin{align*}
 u^\varphi(x) = \left\{\begin{array}{ll}
                 u(x) & x \geq \varphi \\
                 \frac{x u(\varphi)}{\varphi} & x \in [0,\varphi)
                \end{array}
\right..
\end{align*}
\end{enumerate}
In both cases, $\varphi$ should be set very close to 0.

The assumption in Theorem \eqref{th:truncate} and Assumption \ref{ass:2} (ii) requires also the strictly positive lower bound $\epsilon$. This causes problem when applying $u(x) = x^p$, $p > 1$ at the area close to 0. We can again apply above two approximation schemes to overcome the problem by selecting small $\varphi$. In Section \ref{sec:ex}, for both $p > 1$ and $p \in (0,1)$, we apply the second scheme to ensure Assumption \ref{ass:2}.

\subsection{Softmax Policy}
\label{sec:softmax}
Recall that we call a policy is proper, if under such policy every state is visited infinitely often. In this subsection, we show that under some technical assumptions the softmax policy (cf.~Eq.\ \eqref{eq:softmax}) is proper. A policy $\boldsymbol \pi  = [\pi_0, \pi_1, \ldots]$ is deterministic if for all state $s$ and $t$, there exists an action $a \in \mathbf A(s)$ such that $\pi_t(a|s) = 1$. Under one policy $\boldsymbol \pi$, the $n$-step transition probability $P^{\boldsymbol \pi}(S_{n} = s'| S_0 = s)$ for some $s, s' \in \mathbb S$ can be calculated as follows
\begin{align*}
 P^{\boldsymbol \pi}(S_{n} = s'| S_0 = s) = \sum_{S_1, S_2, \ldots, S_{n-1}} P^{\pi_0} (S_1|s) P^{\pi_1} (S_2|S_1) \ldots P^{\pi_{n-1}}(s'|S_{n-1})
\end{align*}
where $P^\pi(y|x) := \sum_{a} \mathcal P(y|x,a)\pi(a|x)$ and $\mathcal P$ is the transition kernel of the underlying MDP. 

\begin{proposition}
 Assume that the state and action space are finite and the assumptions required by Theorem \ref{th:ql} hold. Assume further that for each $s, s' \in \mathbf S$, there exist a deterministic policy $\boldsymbol \pi_d$, $n \in \mathbb N$ and a positive $\epsilon >0$ such that $P^{\boldsymbol \pi_d}(S_{n} = s'| S_0 = s) > \epsilon$. Then the softmax policy stated in Eq.~\eqref{eq:softmax} is proper.  
\end{proposition}
\begin{proof}
 Due to the contraction property of $Q$ (see Lemma \ref{lm:cont}), $\{Q_t\}$ is uniformly bounded w.r.t.~$t$.  Let $\boldsymbol \pi_s = [\pi_0, \pi_1, \ldots]$ be a softmax policy associated with $\{Q_t\}$. Then, by the definition of softmax policies (see Eq.~\eqref{eq:softmax}), there exists a positive $\epsilon_0 > 0$ such that $\pi_t(a|s) \geq \epsilon_0$ holds for each $(s,a) \in \mathbf K$ and $t \in \mathbb N$. It implies that for each $s,s' \in \mathbf S$, 
 \begin{align*}
  P^{\boldsymbol \pi_s}(S_{n} = s'| S_0 = s) \geq \epsilon_0^n P^{\boldsymbol \pi_d}(S_{n} = s'| S_0 = s),
 \end{align*}
 for any deterministic policy $\boldsymbol \pi_d$. Then by the assumption of this proposition, we obtain that for each $s,s' \in \mathbf S$, $P^{\boldsymbol \pi_s}(S_{n} = s'| S_0 = s) \geq \epsilon_0^n \epsilon > 0$. It implies that each state will be visited infinitely often.
\end{proof}

The MDP applied in the behavioral experiment in Section \ref{sec:ex} satisfies above assumptions, since for each $s, s' \in \mathbf S$, there exists a deterministic policy $\boldsymbol \pi_d$ such that $P^{\boldsymbol \pi_d}(S_{n} = s'| S_0 = s) = 1$, $n\leq 4$, no matter which initial state $s$ we start with. 

\section{Magnetic Resonance Protocol and Data Processing}
\label{sec:mr}
Magnetic resonance (MR) images were acquired with a 3T whole-body MR system (Magnetom TIM Trio,
Siemens Healthcare) using a 32-channel receive-only head coil. Structural MRI were acquired
with a T1 weighted magnetization-prepared rapid gra\-dient-echo (MPRAGE) sequence with a
voxel resolution of $1\times 1 \times 1 \textrm{ mm}^3$, coronal orientation, phase-encoding in left-right direction, FoV
= $192 \times 256$ mm, 240 slices, 1100 ms inversion time, TE = 2.98 ms, TR = 2300 ms, and 90 flip
angle. Functional MRI time series were recorded using a T2* GRAPPA EPI sequence with TR =
2380 ms, TE = 25 ms, anterior-posterior phase encode, 40 slices acquired in descending (non-
interleaved) axial plane with $2\times 2 \times 2 \textrm{ mm}^3$ voxels ($204 \times 204$ mm FoV; skip factor = .5), with an
acquisition time of approximately 8 minutes per scanning run.

Structural and functional magnetic resonance image analyzes were conducted in SPM8 (Wellcome
Department of Cognitive Neurology, London, UK). Anatomical images were segmented and
transformed to Montreal Neurological Institute (MNI) standard space, and a group average T1
custom anatomical template image was generated using DARTEL. Functional images were
corrected for slice-timing acquisition offsets, realigned and corrected for the interaction of
motion and distortion using unwarp toolbox, co-registered to anatomical images and transformed
to MNI space using DARTEL, and finally smoothed with an 8 mm FWHM isotropic Gaussian
kernel.

Functional images were analyzed using the general linear model (GLM) implemented in SPM8. First level analyzes included onset regressors for each stimulus event excluding the anticipation phase (see Fig.\ \ref{fig:mdp}a), and a set of parametric modulators corresponding to trial-specific task outcome variables and computational model parameters. Trial-specific task outcome variables (and their corresponding stimulus event) include the choice value of the investment (choice phase) and the total value of rewards (gains/losses) over each round (corresponding to multi-trial feedback event). Model derived parametric modulators included the time series of Q values for the selected action (choice phase), TD (outcome phase). Reward value was not modeled as a parametric modulator because the TD error time series and trial-by-trial reward values were strongly correlated (all rs $> .7$; ps $<.001$). The configuration of the first-level GLM regressors for the standard Q-learning model was identical to that employed in the risk-sensitive Q-learning model. All regressors were convolved with a canonical hemodynamic response function. Prior to model estimation, coincident parametric modulators were serially orthogonalized as implemented in SPM (i.e., the Q-value regressor was orthogonalized with respect to the choice
value regressor). In addition, we included a set of regressors for each participant to censor EPI
images with large, head movement related spikes in the global mean. These first level beta
values were averaged across participants and tested against zero with a t-test. Monte Carlo simulations determined that a cluster of more than 125 contiguous voxels with a single-voxel threshold of $p<.001$ achieved a corrected $p$-value of $.05$.

\end{document}